\documentclass[letterpaper]{article} 
\usepackage[]{aaai2026}  
\usepackage{times}  
\usepackage{helvet}  
\usepackage{courier}  
\usepackage[hyphens]{url}  
\usepackage{graphicx} 
\usepackage{natbib}  
\usepackage{caption} 
\usepackage{algorithm}
\usepackage{dsfont}
\usepackage{algpseudocode}
\usepackage{comment}
\usepackage{amssymb}
\usepackage{booktabs} 
\usepackage{tikz} 
\usepackage{mathtools}
\usepackage{subcaption}
\usepackage{siunitx}
\usepackage{array}
\usepackage{makecell}
\usepackage{enumitem}

\usepackage{bm}
\usepackage{booktabs, makecell}
\newcommand\headercell[1]{%
   \smash[b]{\begin{tabular}[t]{@{}c@{}} #1 \end{tabular}}}

\usepackage{amsthm}
\newtheorem{theorem}{Theorem}

\newtheorem{lemma}{Lemma}
\newtheorem{assumption}{Assumption}

\usepackage{newfloat}
\usepackage{listings}
\DeclareCaptionStyle{ruled}{labelfont=normalfont,labelsep=colon,strut=off} 
\floatstyle{ruled}
\newfloat{listing}{tb}{lst}{}
\floatname{listing}{Listing}
\usepackage[switch, modulo]{lineno}

\newcount\Comments  
\long\def\af#1{{\ifnum\Comments=1\color{blue} [AF: #1]\fi}}
\long\def\ns#1{{\ifnum\Comments=1\color{red} [NS: #1]\fi}}
\long\def\ef#1{{\ifnum\Comments=1\color{cyan} [EF: #1]\fi}}
\long\def\plan#1{{\ifnum\Comments=1\color{purple} [PLAN: #1]\fi}}

\title{A Parallel CPU-GPU Framework for Batching Heuristic Operations in Depth-First Heuristic Search}
\author{
    Ehsan Futuhi\textsuperscript{\rm 1}, 
    Nathan R. Sturtevant\textsuperscript{\rm 1,2}
}
\affiliations{
    \textsuperscript{\rm 1}University of Alberta, 
    \textsuperscript{\rm 2}Alberta Machine Intelligence Institute (Amii)\\
    \{futuhi, nathanst\}@ualberta.ca
}

\begin{document}

\maketitle

\begin{abstract}

The rapid advancement of GPU technology has unlocked powerful parallel processing capabilities, creating new opportunities to enhance classic search algorithms. This hardware has been exploited in best-first search algorithms with neural network-based heuristics by creating batched versions of A* and Weighted A* that delay heuristic evaluation until sufficiently many states can be evaluated in parallel on the GPU. But, research has not addressed how depth-first algorithms like IDA* or Budgeted Tree Search (BTS) can have their heuristic computations batched. This is more complicated in a tree search, because progress in the search tree is blocked until heuristic evaluations are complete. In this paper we show that GPU parallelization of heuristics can be effectively performed when the tree search is parallelized on the CPU while heuristic evaluations are parallelized on the GPU. We develop a parallelized cost-bounded depth-first search (CB-DFS) framework that can be applied to both IDA* and BTS, significantly improving their performance. We demonstrate the strength of the approach on the 3x3 Rubik's Cube and the 4x4 sliding tile puzzle (STP) with both classifier-based and regression-based heuristics.
\end{abstract}

\section{Introduction}

There has been significant recent growth in computational resources, particularly in GPUs \cite{dally2021evolution,rotem2022intel}. GPUs have become indispensable for computation-intensive tasks due to their massive parallelism, capable of performing millions of operations simultaneously. Modern CPUs also continue to evolve with enhanced parallel processing capabilities, enabling faster execution of complex algorithms. This advancement in both CPU and GPU technologies has been well exploited in many fields of Artificial Intelligence, especially deep learning \cite{schrittwieser2020mastering,yao2024tree}.

In classical search algorithms, several approaches \cite{li2022optimal, zhou2015massively, agostinelli2019solving,agostinelli2024q} have been developed to enhance search using the parallel processing capabilities of modern GPUs. 
One potential use of deep learning on GPUs is to learn heuristics to guide search. For instance, \citet{li2022optimal} introduced admissible neural network heuristics that compresses a large pattern database (PDB) heuristic with less information loss than standard compression techniques \cite{felner2007compressed, helmert2017variable}. 

A* \cite{hart1968formal} and Weighted A* \cite{pohl1970heuristic} have had batch versions developed which can use GPU-based heuristics more efficiently for optimal \cite{li2022optimal} and sub-optimal search \cite{agostinelli2019solving}.
The batched versions collect states into batches to evaluate their heuristics in a single parallel neural network lookup using the GPU. This technique, called \emph{batch heuristic evaluations}, utilizes GPU parallelism and significantly improves performance over performing individual neural network lookups for each state.

Algorithms like IDA* \cite{korf1985depth} have not yet had batched variants built. The depth-first nature of IDA* search complicates the batching process, because in a depth-first search there are often only a few states available at one time, while batching is most efficient when hundreds of states are available for computing heuristic values in parallel. The same issue applies to algorithms such as Budgeted Tree Search (BTS) \cite{DBLP:conf/ijcai/HelmertLLOS19,sturtevant2020btsguide}, as both IDA* and BTS are built upon performing repeated cost-bounded depth-first searches (CB-DFS).




This paper addresses the issue by observing that approaches used for CPU parallelization of CB-DFS make many more states available for batching, and thus enable efficient use of GPU-based heuristics. By re-designing CB-DFS for parallel search, we can then effectively build batch versions of IDA* and BTS.
The effectiveness of the batched versions of IDA* and BTS is shown using both regression-based and classifier-based heuristics on the Rubik's Cube and 15-puzzle domains. Batching is highly effective in improving the performance of these algorithms, resulting in over a $40\times$ improvement in search speed. This improvement in search performance opens the door for further research on improving heuristic quality, and for designing new search algorithms to handle the inadmissible values that we expect to find in large neural network-based heuristics.

\section{Background and Related Work}

In \emph{heuristic search}, the broad task is to find a path in a graph $\{G=\{V,E\},s,g,c,h\}$ from a start state $s \in V$ to a goal state $g \in V$, where $c:E \rightarrow \mathbb{R}^{+}$ is a cost function associated with the edges between states. The heuristic function $h(v)$ provides an estimate of the distance from a state $v$ to the goal $g$. The heuristic is considered \emph{admissible} if, for all states $v$, $h(v)$ does not exceed the true shortest distance $h^{*}(v)$ to the goal. It is \emph{consistent} if, for any two states $a$ and $b$, the heuristic satisfies $h(a) \leq c(a,b)+h(b)$. In large state spaces, the graph $G$ is represented implicitly, meaning it is generated online by expanding states and exploring their neighbors. A* and IDA* are guaranteed to find optimal solutions when the heuristic is admissible \cite{felner2011inconsistent}.


\subsection{Heuristics}

Pattern Database (PDB) heuristics \cite{culberson1998pattern} and the related merge and shrink framework \cite{helmert2007flexible} are widely utilized, particularly in problems that exhibit exponential growth \cite{Gnad_Sievers_Torralba_2023}. These heuristics abstract the original graph $V$ into a reduced state space $\phi(V)$. In this abstract space, edges between vertices in the original graph are preserved in the abstracted graph, meaning if an edge exists between $v_{1}$ and $v_{2}$ in $V$, a corresponding edge will exist between $\phi(v_{1})$ and $\phi(v_{2})$ in $\phi(V)$. As a result, abstract distances are admissible estimates of distances in $V$. PDBs can reduce the size of the state space exponentially with only a small loss in heuristic accuracy \cite{felner2009abstraction}.

Standard PDB compression methods \cite{felner2007compressed, helmert2017variable} primarily treat the PDB as a table of numbers. These methods group entries to reduce the PDB's size and replace each entry in a group with the smallest value in that group to ensure the heuristic remains admissible. One common method is \emph{DIV}, where k adjacent entries are grouped by dividing the index by $k$. Another method is \emph{MOD}, which combines entries offset by $\frac{m}{k}$ in a PDB with $m$ total entries using the modulo operator. 

Neural networks have been used to compress PDB heuristics. ADP \cite{samadi2008compressing} used a range of techniques to ensure admissibility. These included a unique loss function to penalize overestimation, a decision tree to partition states, and employing ANNs only for the resulting subsets of states. Any states with inadmissible heuristics were then stored in a hash table. ADP was developed prior to current hardware and is orthogonal to work in this paper on designing more efficient algorithms that use these heuristics.


\citet{li2022optimal} also studied approaches for learning admissible heuristics. They treated the learning of heuristics as a classification problem rather than using regression, because in NP-complete problems the solution length is polynomial, meaning a small number of heuristic values (classes), while the state space is exponential. Admissibility is guaranteed in two ways. First, because the heuristic classes are ordered, the classification quantile used for the predicted class can be adjusted to ensure admissibility. Second, an ensemble of neural networks can be trained, with the minimum value from the ensemble used as the prediction. This technique leverages the diversity of the ensemble to produce an admissible heuristic.



\subsection{Search Algorithms}
A variety of search algorithms can be used for solving shortest path problems. 

{\bf IDA*} combines a cost-bounded depth-first search (CB-DFS) with iterative deepening to find optimal solutions. IDA* performs a series of depth-first searches, each with an increasing cost threshold, which is determined by the current path cost and the heuristic estimate to the goal. IDA* increases the cost threshold conservatively so it can terminate once the goal is found. IDA* is memory-efficient, as it only requires storage for the current path. However, if the number of nodes in each iteration does not grow fast enough, the sum of costs of the iterations may outweigh the cost of the final iteration, increasing the total expansions to $O(N^2)$. 

{\bf BTS} \cite{DBLP:conf/ijcai/HelmertLLOS19} is identical to IDA* when the iterations grow by a suitable constant factor, but reduces the worst-case overhead to $O(N \log C^*)$, where $C^*$ is the optimal solution cost. As in IDA*, BTS repeatedly invokes CB-DFS, just with more aggressive thresholds and additional node expansion limits. Thus, both IDA* and BTS can be improved through GPU parallelization of CB-DFS.



{\bf AIDA*} is designed to parallelize the CB-DFS portion of IDA*. While some forms of depth-first search cannot be parallelized well \cite{reif1985depth}, AIDA* exploits the structure of the problem to create independent subproblems. AIDA* 
has three phases for CPU parallelization of the CB-DFS portion of IDA*, as shown in Figure \ref{fig:AIDA*}:

    
    \subsubsection{Initial Data Partitioning} In this phase
    a single search is performed from the root of the search tree where the leaves of the search have cost greater than the current cost threshold. The threshold is increased until there are sufficiently many leaves above the cost threshold to efficiently perform the following algorithmic steps. At the end of this phase, duplicate nodes are eliminated from the frontier nodes.
    
    \subsubsection{Distributed Node Assignment}
    Next, the leaf nodes from the data partitioning are put into a shared work queue for distributed assignment of work. Path information is maintained to prevent the search from returning to the parents.
    
    
    \subsubsection{Distributed CPU Search} In this phase, processors independently search the {\em subtrees} found below the states in the shared work queue, maintaining the lowest unexpanded $f$-cost in the search. In each iteration either a solution is found and the search terminates, or no solution is found, and the search is repeated at the next cost threshold.
    Since the work queue is shared among all threads, idle threads dynamically retrieve unprocessed work from the queue as they complete their current work.
    
    

\begin{figure}[t]
    \centering
    \includegraphics[width=0.75\linewidth]{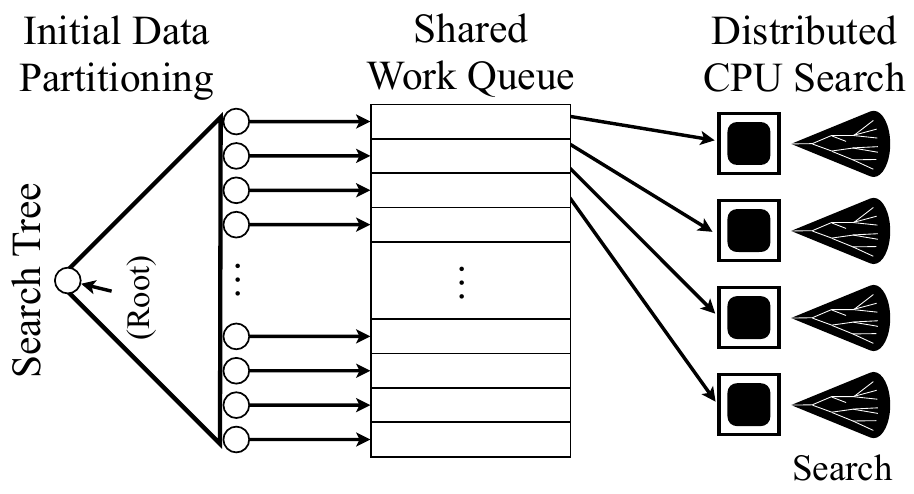}
    \caption{The structure of AIDA*.}
    \label{fig:AIDA*}
\end{figure}


\subsection{GPU Architecture}

GPU parallelization is fundamentally different than CPU parallelization. While in CPU parallelization we need independent subtasks, GPU parallelization works best with correlated tasks that are solved with the same instructions but different data.
GPUs use a hierarchical execution model that enhances parallel processing efficiency. At the heart of this model is the \emph{kernel}, which executes across multiple thread blocks. Each block contains warps, groups of threads that execute the same instructions in parallel. Threads within a block can share data through on-chip memory, but blocks themselves work independently.

The architecture of GPUs is designed to support this model, with each GPU featuring multiple \emph{Streaming Multiprocessors} (SMs). These SMs include on-chip memory, \emph{shader cores}, and \emph{warp schedulers}. Shader cores handle arithmetic and logic operations, while warp schedulers manage the execution of warps, selecting which ones are ready to execute in each cycle. GPUs can be connected to systems either via the PCI-E bus, as in Ubuntu servers, or integrated on the same processor package as the CPU, like in Apple’s M1 or M2 chips. When connected through PCI-E, GPUs typically have dedicated memory, necessitating explicit data transfers between CPU and GPU memory.




\subsection{GPU Parallelization of Algorithms}

In recent years, several parallel versions of the IDA* algorithm have been introduced. AIDA* \cite{reinefeld1994aida}, a highly parallel iterative-deepening search algorithm, was designed for large-scale asynchronous Multiple Instruction, Multiple Data (MIMD) systems. The algorithm partitions the search space and processes it asynchronously across multiple CPU processors. Taking a different approach, \citet{horie2017block} investigate the parallelization of IDA* on GPUs using a block-based approach. The proposed Block-Parallel IDA* (BPIDA*) assigns subtrees to blocks of threads that execute on the same \emph{streaming multiprocessors (SMs)}. BPIDA* takes advantage of local shared memory of within a SM to reduce warp divergence and improve load balancing. 




GPU parallelism has also been effectively utilized in other search algorithms. \citet{zhou2015massively} introduced the first parallel variant of the A* search algorithm called GA* that leverages the computational power of GPUs. GA* uses multiple parallel priority queues to manage the Open list, enabling the simultaneous extraction and expansion of nodes across GPU threads. The heuristic computations are also parallelized across the GPU cores to optimize performance further. GA* is up to 45x faster than a traditional CPU-based A* implementations in large and complex search spaces. Q* search \cite{agostinelli2024q} employs deep Q-networks to calculate combined transition costs and heuristic values for child nodes in a single forward pass, thereby eliminating the need to explicitly generate them. 


\citet{edelkamp2009parallel} investigate bitvector-based search algorithms, where the GPU's parallel processing capabilities allow for the efficient handling of state expansion and duplicate detection. The GPU is also employed for ranking and unranking permutations, computing hash functions, and managing the search frontier, all of which are parallelized to exploit the GPU's architecture. Meanwhile, other approaches have formulated A* and Weighted A* algorithms as differentiable and end-to-end trainable neural network planners \cite{yonetani2021path, archetti2021neural}. These data-driven approaches, rather than learning the heuristic function, take a raw image as an input and convert it to a guidance map by an encoder. 
GPUs have also been instrumental in learning heuristics. For example, \citet{li2022optimal} developed an admissible heuristic for the sliding tile puzzle (STP) and TopSpin using an ensemble of neural networks and classifier quantiles. Other approaches \cite{agostinelli2021obtaining, arfaee2011learning, thayer2011learning, pandy2022learning} have also focused on heuristic learning, though without guaranteeing admissibility. 




\section{CB-DFS for Neural Heuristics}

In a heuristic search problem, the heuristic, $h$, can come from any source. This paper studies the {\em neural heuristic search} problem. In this problem, $h$ is a neural heuristic. This means that $h \in H_{NN}$, where $H_{NN}$ is the set of all heuristics that are computed by a neural network. In recent work, heuristics in $H_{NN}$ have been learned from PDB heuristics \cite{li2022optimal}, and general techniques for heuristic learning have been described \cite{khandelwal2024towards}.
The aim of this paper is to improve algorithms that make use of such heuristics.
Experimental results evaluate the quality of heuristics we currently have access to, but we are working under the assumption that neural network heuristics will continue to improve, and thus improved algorithms will be broadly beneficial.

We now introduce the SingleGPU Batch CB-DFS algorithm, which can be used with both IDA* to create Batch IDA* and BTS to create Batch BTS.
We provide pseudocode and prove theoretical properties. Then, we discuss the MultiGPU CB-DFS algorithm. 

\begin{figure}[t]
    \centering
    \includegraphics[width=0.7\linewidth]{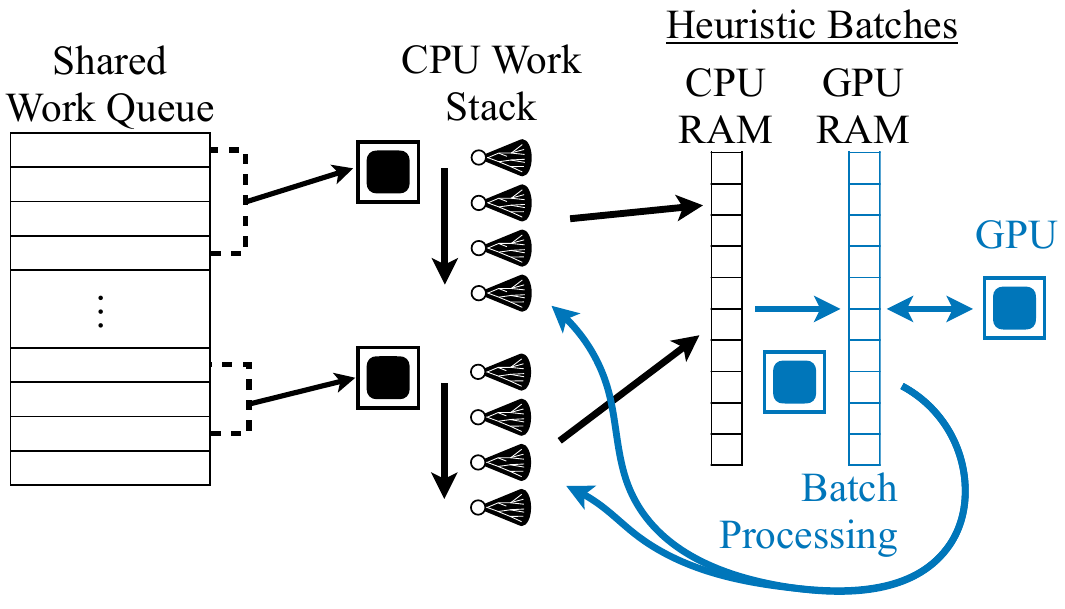}
    \caption{The structure of SingleGPU CB-DFS.}
    \label{fig: SIngleGPU IDA*}
\end{figure}

\begin{algorithm}[t] \small
\caption{Batch IDA*}
\label{alg: batch IDA}   
\begin{algorithmic}[1] 
\State \textbf{Input:} $h_{M}$, start $s$, goal $g$, $d_{init}$
\State works, history, batch  $\gets \{ \}$ 
\State bound $\gets$ $h_{M}(s,g)$
\State foundSolution $\gets$ false
\State \textbf{GenerateWork}($s$,$d_{init}$,history) \label{com: generate work}
\State start a batch-processing CPU thread executing  \textbf{ProcessBatch}()      
\While{not foundSolution}
    \State start search CPU threads executing \textbf{CB-DFS}(bound) \label{code: search}
    \State \textbf{wait} for CB-DFS threads to end
    \State \textbf{UpdateThreshold}(bound)      \label{code: update bound}
\EndWhile   
\end{algorithmic}
\end{algorithm}


\begin{algorithm}[t] \small
\caption{Parallel CB-DFS}
\label{alg: parallel search}   
\begin{algorithmic}[1] 
\Function{}{}\textbf{CB-DFS}(bound)
    \State \textbf{Initiate} stack[workNum]
    \State \textbf{Initialize} terminated $\gets$ [0, 0, \dots, 0] \textbf{of length} workNum
    \State counter $\gets$ 0
    \State miss $\gets$ 0
    \For{$i = 1$ \textbf{to} workNum}               \label{code: stacks filling starts}
        \State stack[i] $\gets$ works.pop() 
    \EndFor                                         \label{code: stacks filling ends}
    \While{miss $<$ workNum} \hfill \ \ \   \ \ \ \ \ \ \  $\triangleright$ wait for all works to 
        \If{stack[counter] is \textit{done}} \ \ \ \ \ \ \ \ \ \ \  \ \ \ \ \ be done
            \If{works is not \textit{empty}}
                \State stack[counter] $\gets$ works.pop()   \label{code: new work}
            \Else
                \State miss $\gets$miss+1
                \State terminated[counter] $\gets$ 1
            \EndIf
        \EndIf
        \If{not terminated[counter]}
            \State \textbf{DoIteration}(stack[counter],bound)
        \EndIf
        \State counter $\gets$ counter+1
    \EndWhile
\EndFunction   
\end{algorithmic}
\end{algorithm}

\subsection{SingleGPU Batch CB-DFS}

Batch heuristic lookups in best-first search algorithms like A* are relatively easy because once the search gets started, there are many states waiting in the open list. But, in a cost-bounded depth-first tree, the depth-first search cannot continue until the $f$-cost of all children is known. Thus, to get sufficient states for batching heuristic lookups, we must either have a large enough branching factor to permit efficient batching, search speculatively beyond the cost limit, or run several CB-DFS searches in parallel so we can batch states across subtrees. This paper explores the last approach.

The overall structure of \emph{SingleGPU CB-DFS} is illustrated in Figure \ref{fig: SIngleGPU IDA*}. It begins similarly to AIDA* by generating a {\em Shared Work Queue} with subtrees that can be searched independently. However, instead of having each CPU take a single subtree, the CPUs now fill a work stack with multiple subtrees. A single expansion is performed on the first subtree, which results in a set of children that need their heuristics evaluated before the search can continue. These states are placed into a shared {\em heuristic batch queue} on the CPU, and then work continues on the next subtree in the work stack, generating more states until the heuristic batch queue is full. If multiple CPUs are available, they can execute this process in parallel.

Once the heuristic batch queue is full, a dedicated {\em batch-processing} thread manages the batch evaluation of states. This involves (1) copying the data to GPU memory, (2) evaluating the batch on the GPU using the model $h_{M}$, and then (3) returning the resulting heuristics to each of the CPUs. 
%
The batch is then cleared, allowing the next set of generated states to be processed. More than one batch is maintained in RAM so the CPUs can continue to search in parallel to the batch evaluation.
%
%
The CB-DFS completes when the shared work queue is exhausted and
all threads have completed any searches remaining in their work stacks. This completes the high-level CB-DFS.

The pseudo-code showing batch CB-DFS and how it is integrated into SingleGPU Batch IDA* is outlined in Algorithm \ref{alg: batch IDA}.
The \texttt{GenerateWork} function (line \ref{com: generate work}) is called once before the main search loop begins, instructing an initial search tree to depth $d_{init}$.
BatchIDA* then invokes \texttt{CB-DFS} in parallel across CPU threads given the cost threshold of the high-level search (line \ref{code: search}). The cost threshold is updated (\texttt{UpdateThreshold} function) after the CB-DFS with the current threshold is completed (line \ref{code: update bound}).




The \texttt{DoIteration} function (Algorithm \ref{alg: do iteration}) expands subtrees one at a time. While executing the \texttt{DoIteration} function, if a subtree cannot progress further---meaning the heuristic evaluation is not yet available for the top node (line \ref{code: switch to another work}) -- the thread switches to the next subtree in its stack. Upon revisiting the same subtree, the thread resumes from the same node if its heuristic evaluation is ready; otherwise, it switches again to the next subtree. 

The batch processing is done via the \texttt{ProcessBatch} function, which is not shown. The most important detail of this algorithm is that it has a timeout after which a batch is evalauted even if it is not completely full. This is important when very little work is available at the end of an iteration, and impacts the integration with BTS.


\begin{algorithm}[tb] \small
\caption{Subtree expansion}
\label{alg: do iteration}   
\begin{algorithmic}[1] 
\Function{}{}\textbf{DoIteration}(work, bound)
    \State newStatesFound $\gets$ false
    \While{not newStatesFound}
        \State $s \gets$ work.GetTop()
        \If{$h_{M}(s)$ is not \textit{ready}} \label{code: switch to another work}
            \State return
        \EndIf
        \If{$h_{M}$($s$,goal)$<$bound}
            \State newStatesFound$\gets$true
        \EndIf
    \EndWhile
    \State actions $\gets$ env.GetActions($s$)
    \For{each action in actions}
        \State $s_{next}$ $\gets$ env.ApplyAction($s$, action)
        \State work.add($s_{next}$) 
        \State batch.add(\textbf{TensorRepresentation}($(s_{next})$))  \label{code: representation}
    \EndFor
\EndFunction
\end{algorithmic}
\end{algorithm}


Overall, the performance of Batch IDA* hinges on the balance between two key events: \textit{filling the batch} and \textit{processing the batch}. The time required to fill the batch depends on the speed of the CPU CB-DFS threads, while processing time is determined by the GPU's efficiency and the transfer time via the PCI-E bus. If batch-filling is faster, CB-DFS threads may remain idle, either waiting to add new nodes or for heuristic evaluations. Conversely, if batch processing is faster, the GPU is underutilized. Therefore, these two processes must be balanced to ensure optimal resource utilization during the CB-DFS.

\subsection{Batch CB-DFS with Multiple GPUs}

In most servers, multiple GPU devices are available, although the design shown thus far only incorporates a single GPU. If there are many fast CPUs available, they can end up waiting on the GPU for heuristic evaluations.


We scale Batch CB-DFS to multiple GPUs by loading $h_{M}$ onto each GPU, and running one batch-processing thread per GPU. CB-DFS threads are then uniformly assigned to the available GPUs, with each GPU processing the batch of nodes generated by its respective CB-DFS threads. This setup maximizes CPU utilization by increasing the number of batches while ensuring sufficient GPU resources are available to process them concurrently. 

\subsection{Correctness of Batch CB-DFS}

In this section, IDA*, BTS, and AIDA* are all implemented with a CB-DFS sub-routine. Thus, as long as Batch CB-DFS searches the same tree as a classic CB-DFS search, we can substitute Batch CB-DFS without impacting the correctness of these algorithms. The batching operation does not impact the fact that the cost limit in the search is bounded, because the cost bound is still respected in the subtree expansions. Detailed proofs are found in supplementary material.

\subsection{Batch BTS Enhancements}

One important feature of the BTS algorithm is that it avoids the worst case where the IDA* tree grows too slowly, which can happen when there are many unique $h$-costs (e.g. as a result of using a regression-based heuristic). BTS guarantees that in each iteration the number of nodes will grow exponentially. It does this by establishing a limit of how fast the tree should grow in the next iteration (measured by node expansions), and then does a search over the cost bounds to find the cost bound that achieves this growth. Thus, in Batch BTS we must impose an additional node expansion limit in the CB-DFS. For efficiency, we do not check this after every expansion, but often enough to ensure that the search will terminate within a constant bound of the limit.


Next, we note that using classic heuristics, there is little cost for running many small iterations at the very beginning of search. But, with neural heuristics, doing so is very slow, because the search never grows large enough to exploit large batch sizes. Thus, a very small iteration of BTS can be a thousand times more expensive than a larger iteration.

To mitigate this, we make two additional modifications in Batch BTS. First, instead of setting the initial node limit to a small constant, such as 1, we use a much larger constant -- in our case 10{,}000 expansions. This allows the initial searches to be significantly larger. But, setting the node limit on its own is not sufficient, we also need to raise the cost bound of the first search. Thus, instead of using \(h(\text{root})\), we use \( h(\text{root}) + 1 \) instead. Neither of these changes impact the asymptotic behavior of the algorithm. Because the overheads are just large constants, on sufficiently large problems these modifications would not be necessary. But, in practice, they significantly improve performance.



\subsection{Memory Cost Analysis}

The memory cost of IDA* is $\mathcal{O}(d)$, where $d$ is the solution depth~\cite{korf1985depth}. For AIDA*, which utilizes $n$ CB-DFS threads, the memory cost increases to $\mathcal{O}(dn)$ as it concurrently expands $n$ subtrees. In Batch IDA* and Batch BTS's CB-DFS, each thread manages $k$ subtrees in a stack, resulting in a memory cost of $\mathcal{O}(dnk)$. We know that $n$ and $k$ are related to hardware speed on any domain. So, these parameters are expected to be constant as hardware is constant. Additionally, while the heuristic evaluations for unexpanded nodes are stored temporarily, memory is freed after their expansion, keeping the total memory cost at $\mathcal{O}(dnk)$.

AIDA*, Batch IDA*, and Batch BTS require loading their respective heuristics into memory. If AIDA* uses a PDB heuristic $h$ and Batch IDA* uses a learned model $h_{M}$, the total memory cost for AIDA* is $\mathcal{O}(nd+size(h))$ and for Batch IDA* it is $\mathcal{O}(dnk +size(h_{M}))$. While PDB heuristics typically fit into memory and thus their cost is constant, large PDBs, such as the 12-edge Rubik's cube PDB (500 GB), exceed memory capacity, making them impractical for AIDA*. 
This limitation underscores the importance of further work on compressing PDB heuristics using deep learning and designing search algorithms, such as the CB-DFS used by Batch IDA* and Batch BTS, to use them efficiently.

\begin{table*}[t]
\small
\centering
\begin{tabular}{
    c
    S[table-format=4.0]
    S[table-format=3.2]
    S[table-format=8.0]
    S[table-format=8.0]
    S[table-format=3.2]
    S[table-format=6.0]
    S[table-format=6.0]
}
\toprule
\multicolumn{8}{c}{Average Performance Over 50 Instances} \\
\cmidrule(l){1-8}
\headercell{Algorithm} & {{Batch Size}} &\multicolumn{3}{c} {3x3 Rubik's Cube} &\multicolumn{3}{c} {4x4 STP}\\
\cmidrule(l){3-5}
\cmidrule(l){6-8}
& & {Time (s)} & {Expanded} & {Generated}  & {Time (s)} & {Expanded} & {Generated} \\ 
\midrule
  SingleGPU Batch IDA*  & 1              &  \multicolumn{1}{c}{$>$ 100}      & {-}            & {-}                    & 16.80        & 104869       & 225542      \\
  SingleGPU Batch IDA*  & 8              &  \multicolumn{1}{c}{$>$ 100}      & {-}            & {-}                    & 3.97         & 93610        & 201217      \\
  SingleGPU Batch IDA*  & 80             &  58.07          & 902026          & 11966615                & 0.94         & 93575        & 201145                  \\
  SingleGPU Batch IDA*  & 800            &  5.46           & 908522          & 12054452               & 0.54         & 104945       & 225704                  \\
  SingleGPU Batch IDA*  & 8000           &  3.46           & 900725          & 13734532              & 0.71         & 115190       & 248108                  \\
  2GPU Batch IDA*       & 800            &  2.78           & 982864          & 13044563               & 0.44         & 104833       & 225468                  \\
  2GPU Batch IDA*       & 8000           &  2.51           & 936426          & 13172161              & 0.64         & 115876       & 249264                  \\
  Batch A*              & 1000           &  23.58          & 779495          & 14030893              & 0.18         & 17892        &  41465                  \\
\bottomrule
\end{tabular}
\caption{Summary results comparing Batch IDA* and Batch A* using the same PDB heuristic.}
\label{table: same tree}
\end{table*}



\section{Experimental Results}

The primary purpose of the experimental results is to evaluate whether or not the batched version of CB-DFS can effectively reduce the overhead of neural heuristics through GPU parallelism. We evaluate this in both the  Batch IDA* and Batch BTS algorithms with neural heuristics built on both classifiers and regression, as regression-based heuristics are expected to fail with IDA*, but not BTS. We further evaluate the impact of GPU count and neural heuristic size on performance. Additional experiments examining the effects of hardware and algorithmic hyperparameters are included in the supplementary material.




\subsection{Experimental Setup}

We conduct our experiments across two domains: the 3x3 Rubik's cube and the 4x4 sliding tile puzzle (STP). We use relatively modest computational resources in our experiments. Specifically, we utilize a server equipped with 32 AMD Ryzen Threadripper 2950X CPU cores and two NVIDIA GeForce RTX 2080 Ti GPUs running CUDA version 12.4. The Batch algorithms are implemented in C++, and Libtorch is employed for the deep learning components. Any CPU parallelization uses one thread per CPU core. For batched algorithms, we use a 4-millisecond timeout to trigger batch processing when an insufficient number of states are available to fully populate the batch. For the standard algorithms, IDA* and AIDA*, we use $h_{\text{PDB}}$ which is the 8-corners PDB for Rubik's cube, and the sum of  1-7 and 8-15 tile additive PDBs for STP. For the STP, we use standard benchmark instances \cite{korf1985depth}. For the Rubik's Cube domain, unless otherwise stated, we generated instances of solution length 11 using random walks starting from the goal state. PDB heuristics and baselines are from HOG2\footnotemark{}\footnotetext{https://github.com/nathansttt/hog2/tree/PDB-refactor}.

The heuristic $h_{M}$ for the STP domain is the model introduced by \citet{li2022optimal}. It is constructed by combining two ensemble models: the first is trained on the difference between the 1--7 tile PDB and the Manhattan distance heuristic, while the second is similarly trained using the 8--15 tile PDB. For the Rubik's Cube domain, we trained two admissible heuristics \( h_{M} \), both based on the 8-corner PDB, with a size of 1.9~MB. The first model is a classifier that was trained extensively to achieve both admissibility and the same average heuristic as in the input PDB. The second model is a regression network trained using mean squared error (MSE) loss. For admissibility, we subtracted 2.2 from all predicted heuristic values. However, this adjustment reduced the average heuristic by 2.3 compared to the perfect heuristic.



\subsection{Fixed-Tree Evaluation}

Our first experiment is designed to evaluate the effectiveness of batch operations in Batch IDA* with a classifier-based heuristic. We designed this experiment to isolate as many variables as possible in the experiment and to isolate the quality of neural heuristic from the mechanics of Batch IDA*. We isolate the impact of the neural heuristic training by fixing the tree to only contain states expanded by a baseline PDB heuristic ($h_{\text{PDB}}$). The search evaluates the neural heuristic at each state, but uses $h_{\text{PDB}}$ for pruning.

Under this methodology all algorithms are searching exactly the same tree. This allows us to initialize the neural heuristic randomly and quickly experiment with different network topologies without having to re-train the neural heuristics. This approach can measure the effectiveness of parallelism, but does not measure the quality of any learned heuristics, which is not the primary concern of this paper.

Results of this experiment are presented in Table \ref{table: same tree}. Each algorithm is given 5000s to solve all problems in the problem set, or an average of 100s per problem. Batch IDA* with a batch size of 1 or 8 fails to solve all of the Rubik’s Cube instances within the time limit. But, increasing the batch size to 8000 reduces the average time to 3.46s. Using two GPUs and a batch size of 8000 further reduces this to 2.51s, which is almost ten times faster than Batch A*. Note that Batch A* is only parallel on the GPU, not the CPU. Adding CPU parallelization to Batch A* is an open challenge. 

We use the first 50 instances from Korf's benchmarks for the STP domain. On this domain, going from batch size of 1 to 800 reduces the running time from 16.80s to 0.54s, a 30$\times$ improvement. Using two GPUs and batch size 800 reduces this further to 0.44s. The variance in nodes expanded is due to a deep initial phase in Batch IDA* aimed at producing sufficient work pieces, which results in batches containing extra nodes that are not necessary to solve the problem ($f$-cost greater than the threshold). Although Batch IDA* is faster than Batch A* with respect to time per node expansion, Batch A* expands fewer nodes because it detects duplicates, which Batch IDA* does not.

\sisetup{detect-weight=true, detect-inline-weight=math}

\begin{table}[tb]
\centering
\small
\begin{tabular}{
    l                                
    S[table-format=4.0]              
    l                                 
    S[table-format=9.0]              
}
\toprule
\multicolumn{4}{c}{Average Performance Over 50 RC Instances} \\
\cmidrule(l){1-4}
Algorithm  & {Batch Size} & {Time ± SE (s)} & {Expanded} \\ 
\midrule
\multicolumn{4}{c}{Solution Length 8} \\
\midrule
Batch BTS        & 10       & \hspace{5pt}33.41  ±  5.80     & 537763         \\
Batch BTS        & 100      & \hspace{5pt}4.26   ±  0.54      & 569475         \\
Batch BTS        & 1000     & \hspace{5pt}0.97   ±  0.13      & 571179         \\
Batch BTS        & 2000     & \hspace{5pt}0.83   ±  0.11      & 583135         \\
Batch IDA*       & 1000     & \hspace{5pt}163.20 ±  54.63   & 366530531      \\
Batch IDA*       & 2000     & \hspace{5pt}144.91 ±  37.63   & 360097779      \\
\midrule
\multicolumn{4}{c}{Solution Length 9} \\
\midrule
Batch BTS        & 10       & \hspace{5pt}169.73 ± 65.58   & 1211327        \\
Batch BTS        & 100      & \hspace{5pt}20.25 ± 6.58     & 1217398        \\
Batch BTS        & 1000     & \hspace{5pt}3.64 ± 0.94      & 1294972        \\
Batch BTS        & 2000     & \hspace{5pt}3.01 ± 0.94      & 1294329        \\
Batch IDA*       & 2000     & \hspace{20pt}{-}             & \hspace{20pt}{-}             \\
\bottomrule
\end{tabular}
\caption{Performance on 3x3 Rubik's Cube instances with the regression model.}
\label{table:BatchBTS-vs-BatchIDA}
\end{table}

\begin{table}[tbp]
\small
\centering
\begin{tabular}{
    S[table-format=1.1]
    S[table-format=2.2]
    S[table-format=3.2]
    S[table-format=2.2]
    S[table-format=2.2]
}
\toprule
\multicolumn{5}{c}{Average Time (s) Over 50 Instances} \\
\cmidrule(l){1-5}
\headercell{Model size (MB)}  & \multicolumn{2}{c}{SingleGPU Batch IDA*} & \multicolumn{2}{c}{Batch A*}\\
\cmidrule(l){2-3}
\cmidrule(l){4-5}
& RC & STP & RC & STP \\ 
\midrule
  0.2      & 3.46           &  0.52           & 11.16         & 0.18       \\
  2.3      & 4.75           & 0.78            & 16.71         & 0.29        \\
  13.6     & 10.23          & 1.45            & 20.66         & 0.38         \\  
\bottomrule
\end{tabular}
\caption{Impact of model size on Batch IDA* speed.}
\label{table: model-size analysis}
\end{table}

\subsection{Regression-Based Neural Heuristics}

In Table \ref{table:BatchBTS-vs-BatchIDA} we evaluate the impact of using a real-valued heuristic on Batch IDA* and Batch BTS. In this experiment both algorithms are using an identical neural heuristic, and we are measuring the cost of searching the resulting tree with that algorithm. Because the heuristic is learned using regression, which produces real-valued heuristics, we the IDA* iterations to grow very slowly. This is confirmed by the experimental results. 

Due to the poor performance of Batch IDA*, we only report results where the algorithm is able to solve all instances within the 4-hour time limit. Batch IDA* expands orders of magnitude more states than Batch BTS, reinforcing that Batch BTS is better suited for regression-based heuristics. On instances with solution length 9, Batch IDA* expands billions of nodes per instance and fails complete within the time limit, while Batch BTS successfully solves all instances across all batch sizes. Notably, Batch BTS achieves a $56\times$ speedup when the batch size is increased from 10 to 2000.


\subsection{Model Size}

Our third experiment looks at the impact of model size on CB-DFS performance. We run Batch IDA* on the same domains with a single GPU and the best parameters from Table \ref{table: same tree}. We vary the model size by fixing the number of layers, but changing the number of fully connected neurons in each layer. The results of this experiment are found in Table \ref{table: model-size analysis}. These results show that increasing the size of the model does not linearly increase the time required to evaluate the model, although larger models are more expensive to evaluate. Thus, there will be trade-offs between model quality and speed when deploying neural heuristics.

The results also show that Batch A* is less affected by model size than Batch IDA*. The main reason for this difference is that the batches in Batch IDA* are not always full, meaning the GPU isn't being used as efficiently. This can be caused by node generations above the cost thresholds and the uneven distribution of work between subtrees. When the model requires more time to evaluate, these issues are accentuated. See the Supplementary Material for further analysis.

\begin{table}[tbp]
\small
\resizebox{\columnwidth}{!}{%
\begin{tabular}{
    c
    S[table-format=1.2]
    S[table-format=8.0]
    S[table-format=1.2]
    S[table-format=8.0]
}
\toprule
\multicolumn{5}{c}{Average Performance Over 50 Instances} \\
\cmidrule(l){1-5}
\headercell{Algorithm}  & \multicolumn{2}{c}{RC} & \multicolumn{2}{c}{STP}\\
\cmidrule(l){2-3}
\cmidrule(l){4-5}
& {Time (s)}  & {Generated} & {Time (s)}  & {Generated} \\ 
\midrule
  4GPU Batch IDA*             &  2.56          & 12027388          & 0.57      & 238440                   \\
  2GPU Batch IDA*             &  2.97          & 11322876          & 0.58      & 239492                  \\
  SingleGPU Batch IDA*         &  6.11          & 12498111          & 0.99      & 231958                   \\
  AIDA*                        &  2.26          & 14793061          & 0.02      & 101020                     \\ \bottomrule
\end{tabular}
}
\caption{Impact of GPU counts on Batch IDA* performance.}
\label{table: GPU count}
\end{table}

\begin{table}[tbp]
\small
\resizebox{\columnwidth}{!}{%
\begin{tabular}{
    c
    c
    S[table-format=2.2]
    S[table-format=7.0]
    S[table-format=8.0]
}
\toprule
\multicolumn{5}{c}{Average Performance Over 50 Instances} \\
\cmidrule(l){1-5}
\headercell{\\ Algorithm} &\multicolumn{4}{c} {4*4 STP}\\
\cmidrule(l){2-5}
\cmidrule(l){2-5}
& {batch size} & {Time (s)} & {Expanded} & {Generated} \\ 
\midrule
  SingleGPU Batch IDA*  & 8000           &  13.11          & 777652          & 1598045                   \\
  2GPU Batch IDA*       &  8000          &  11.78          & 770341          & 1428177                   \\
  AIDA*                 & -              &  1.21           & 10513092        & 30254284                   \\
  IDA*                  & -              &  4.23           & 8180928         & 17639025                    \\
  Batch A*              & 1000           &  2.26           & 88977           & 273505                    \\ \bottomrule
\end{tabular}
}
\caption{STP results using a learned admissible heuristic.}
\label{table: learning tree on STP}
\end{table} 

\begin{table}[tbp]
\small
\resizebox{\columnwidth}{!}{%
\begin{tabular}{
    c
    c
    S[table-format=2.2]
    S[table-format=7.0]
    S[table-format=8.0]
}
\toprule
\multicolumn{5}{c}{Average Performance Over 50 Instances} \\
\cmidrule(l){1-5}
\headercell{\\ Algorithm} &\multicolumn{4}{c} {3*3 Rubik’s cube}\\
\cmidrule(l){2-5}
\cmidrule(l){2-5}
& {batch size} & {Time (s)} & {Expanded} & {Generated} \\ 
\midrule
  2GPU Batch IDA*       & 8000           &  5.11      & 1001733          & 13028957                   \\
  SingleGPU Batch IDA*  & 8000           &  8.06      & 1005874          & 13082769                   \\
  AIDA*                 & -              &  2.74      & 2218215          & 29401255                   \\
  IDA*                  & -              &  5.84      & 2106615          & 27900766                    \\ 
  Batch A*              & 1000           &  42.07     & 779491           & 14030761                     \\  \bottomrule
\end{tabular}
}
\caption{RC results using a learned admissible heuristic.}
\label{table: learning tree on RC}
\end{table}

\subsection{GPU Count} 

To evaluate the impact of GPU count, we present the performance of Batch IDA* with different numbers of GPUs in Table \ref{table: GPU count}, using other servers for this experiment. In both domains, performance improves with additional GPUs, with the 4-GPU version matching AIDA*'s performance on the Rubik's Cube. However, the performance gains diminish with increasing GPU counts, due to fewer threads being available to populate batches. Additionally, we assess Batch IDA* using various hardware configurations, with results provided in the Supplementary Material.

\subsection{PDB Comparison}

We conclude by evaluating neural heuristics in Batch IDA* against AIDA*. Batch algorithms rely on admissible neural network heuristic $h_{M}$, while standard algorithms use a compressed PDB heuristic $h_{\text{DIV}}$ of equal size. Results for STP and Rubik's Cube are presented in Table \ref{table: learning tree on STP} and Table \ref{table: learning tree on RC}, respectively. The slower performance of learning algorithms compared to the results in Table \ref{table: same tree} is due to the need for additional tensor operations per node to ensure heuristic admissibility. Standard algorithms expand more nodes in both domains since $h_{\text{DIV}}$ is a weaker heuristic than $h_{M}$. AIDA* performs significantly better than all other algorithms, while among the batch algorithms, Batch IDA* remains faster than Batch A* in terms of constant time per node. We provide a Detailed analysis of this performance gap between AIDA* and Batch IDA* in the Supplementary Material.

\section{Discussion and Conclusions}

This paper has shown how to parallelize CB-DFS on the CPU and GPU for efficiently using neural heuristics during search. CB-DFS is explored with both Batch IDA and Batch BTS. The batching approach provides significant gains when batch sizes reach a sufficiently large size. Thus, this work provides the foundation for future research on building larger and stronger neural heuristics.

\section{Acknowledgements}

This work was supported by the National Science and Engi-
neering Research Council of Canada Discovery Grant Pro-
gram and the Canada CIFAR AI Chairs Program.

\bibliography{aaai25}


\clearpage
\appendix
\section{Appendix}

In this section, we first present the results for the average batch size during runtime. Next, we examine the performance difference between Batch IDA* and AIDA*, considering their similar search procedures. Finally, we provide additional results with various hardware configurations to further assess the robustness of Batch IDA*.

\subsection{Batch Size in Practice}

As discussed earlier, the batch size during runtime is not always consistent. This occurs because CPU threads may not be able to fully populate the batch when there are only a few pieces of work remaining in the queue. Additionally, in the initial iterations with low cost thresholds, a significant portion of nodes are pruned, resulting in fewer nodes being collected within the specified timeout than the target batch size. Figure \ref{fig: batch} shows the average batch size and total number of batches during runtime for both domains. The target batch size is 8000 for the Rubik's cube and 500 for STP. In both domains, as the search space grows, the number of batches and their sizes increase. In the STP domain, the maximum average batch size reaches only about half of the target batch size due to deeper trees, more IDA* iterations, and a lower branching factor. This limitation in batch size prevents full utilization of GPU parallelism, contributing to the reduced speedup observed in STP. 

\begin{figure*}[htbp]
    \centering
    \small
    \begin{subfigure}[b]{0.8\textwidth}
        \centering
        \includegraphics[height=5.4cm]{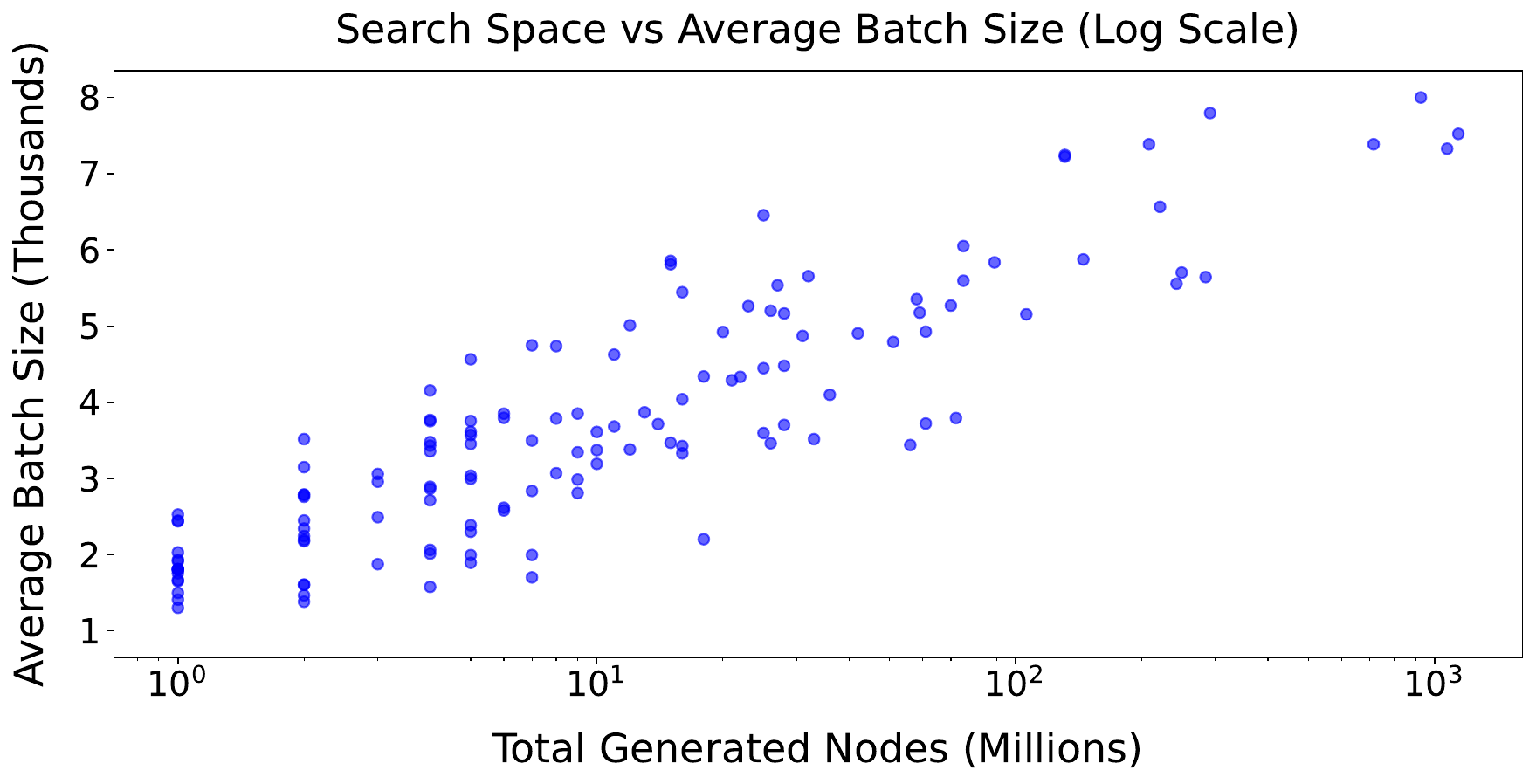}
        \caption{Average batch size in RC}
    \end{subfigure}
    
    \begin{subfigure}[b]{0.8\textwidth}
        \centering
        \vspace{10pt}
        \includegraphics[height=5.4cm]{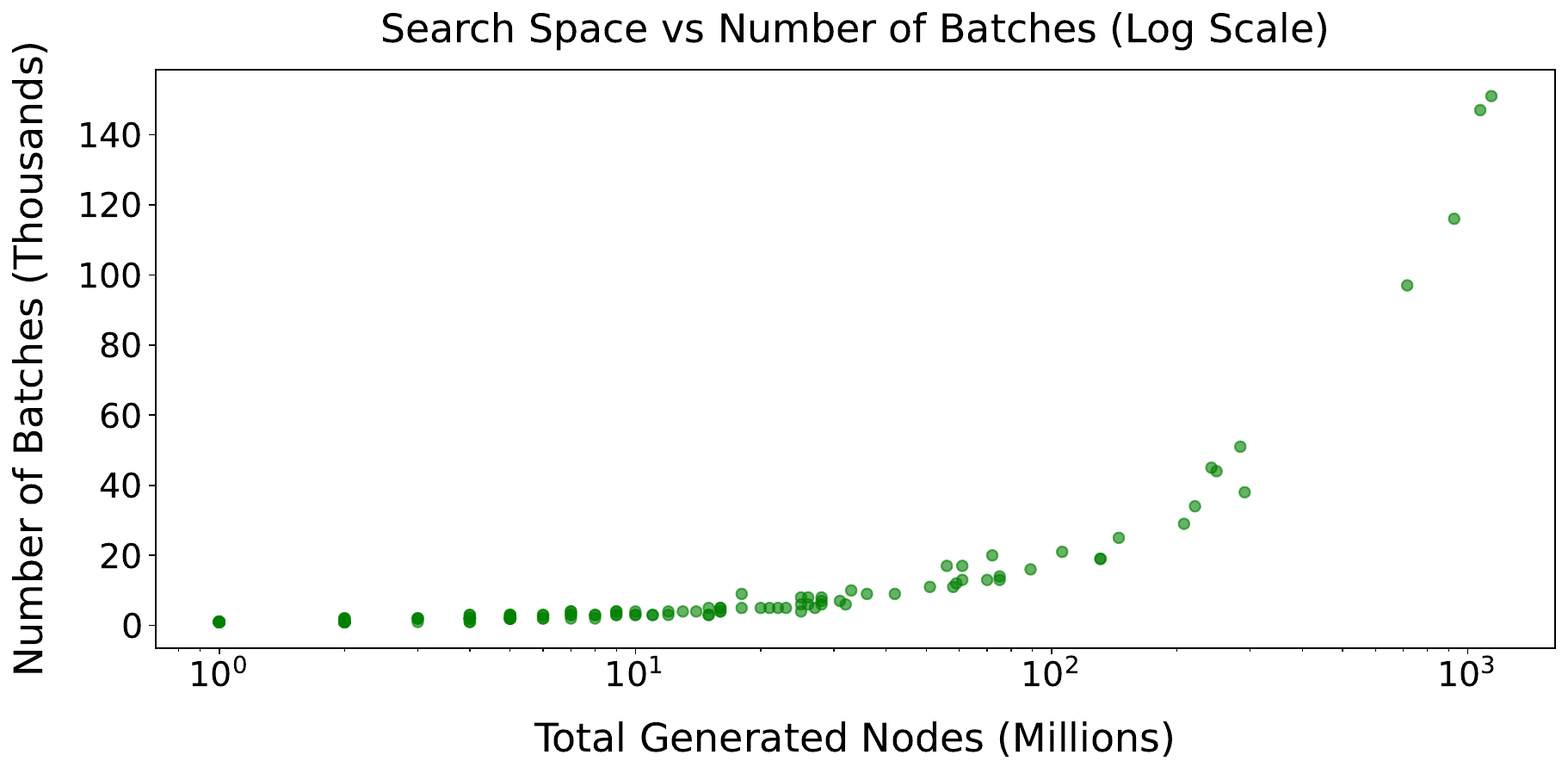}
        \caption{Number of batches in RC}
    \end{subfigure}
    
    \begin{subfigure}[b]{0.8\textwidth}
        \centering
        \vspace{10pt}
        \includegraphics[height=5.4cm]{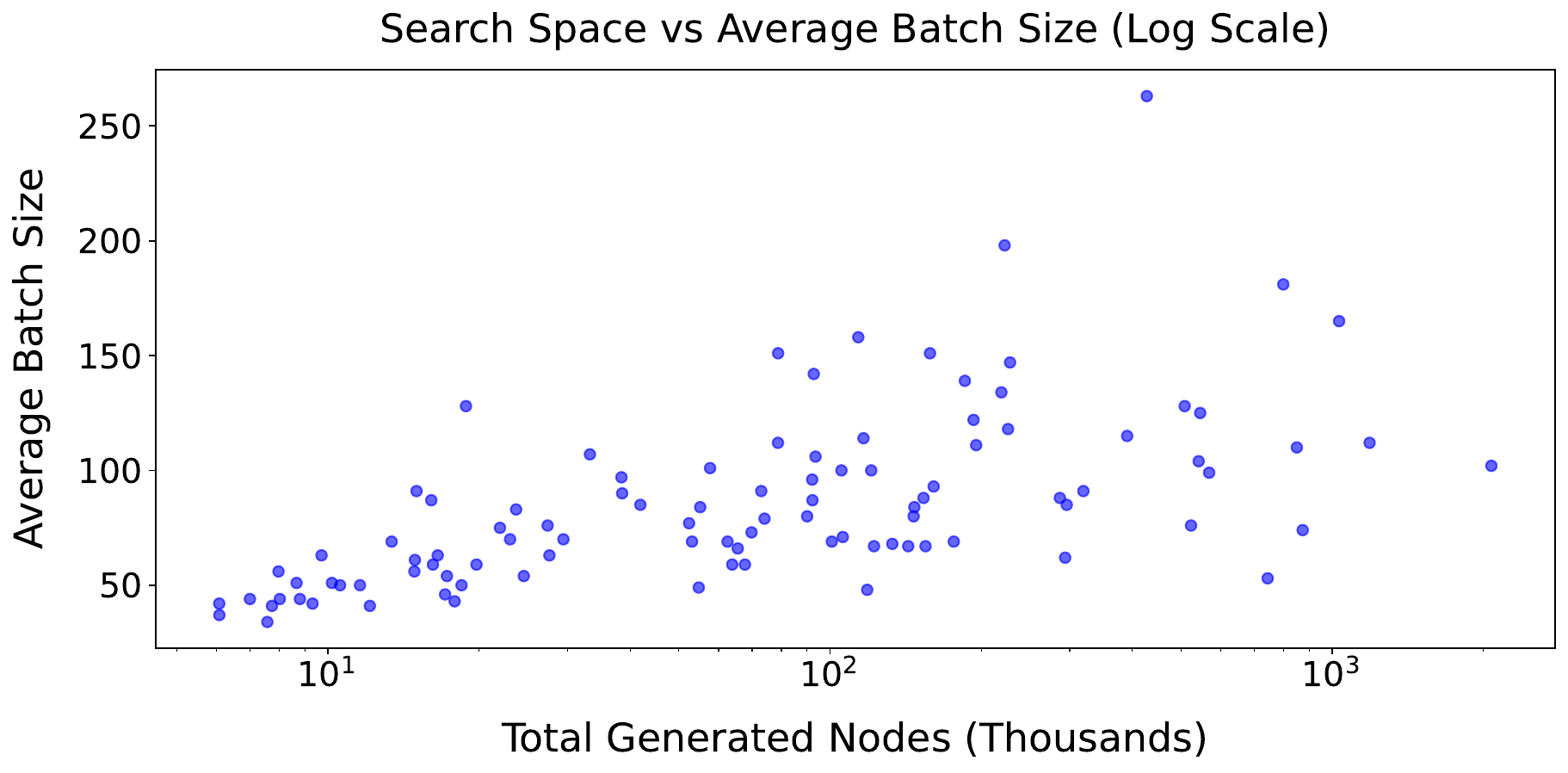}
        \caption{Average batch size in STP}
    \end{subfigure}
    
    \begin{subfigure}[b]{0.8\textwidth}
        \centering
        \vspace{10pt}
        \includegraphics[height=5.4cm]{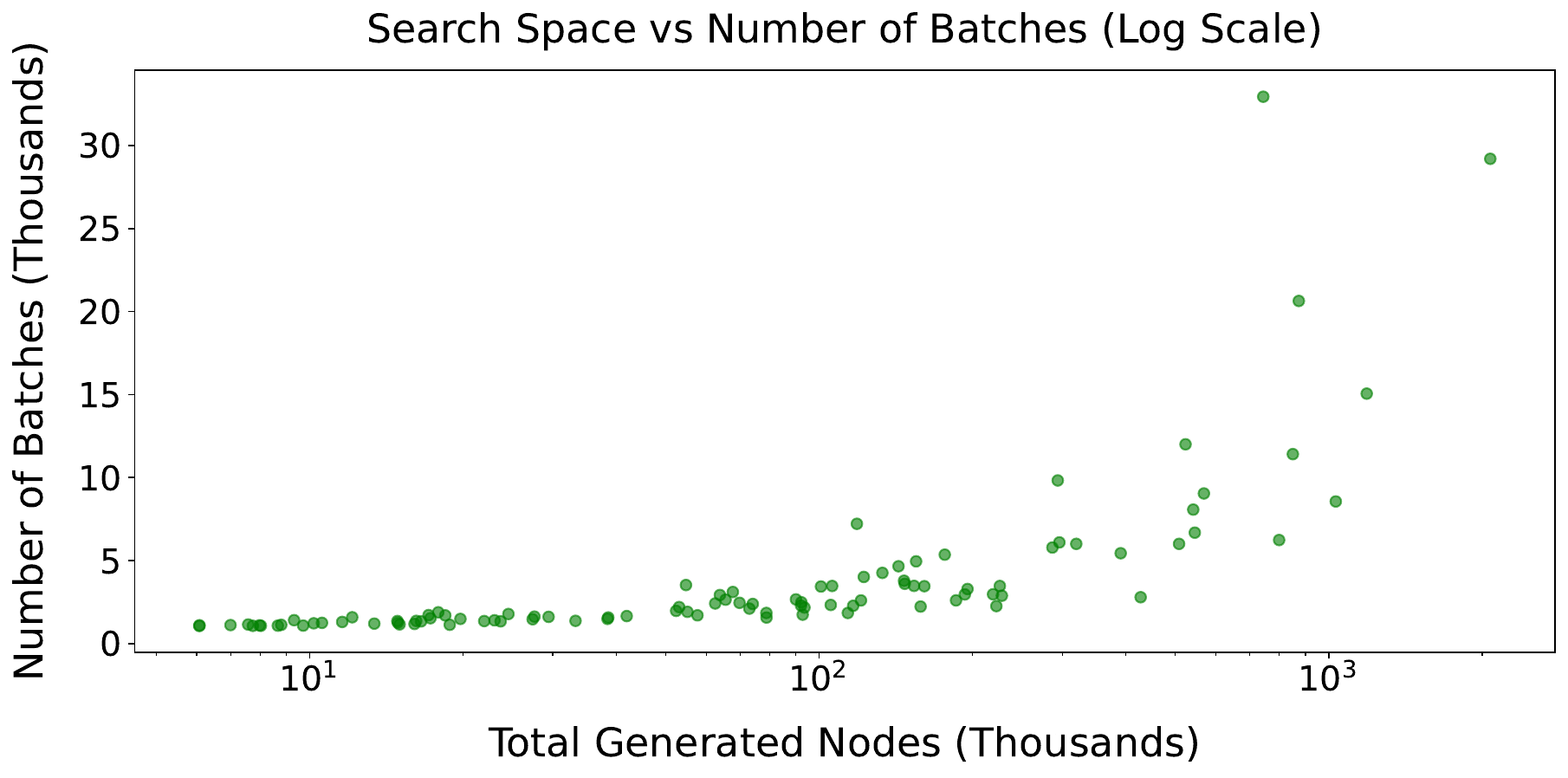}
        \caption{Number of batches in STP}
    \end{subfigure}
    
    \caption{Average batch size and total number of batches during runtime.}
    \label{fig: batch}
\end{figure*}

\subsection{Comparative Analysis of Batch IDA* and AIDA*}

In this section, we analyze the performance gap between Batch IDA* and AIDA*. Let $h_{M}$ represent a neural network heuristic and $h_{\text{PDB}}$ a pattern database (PDB) heuristic. We denote the batch size for $h_{M}$ by $b_{\text{NN}}$. If we assume that we have $n$ CPU threads, AIDA* also uses batch strategy as all threads perform PDB lookups in parallel. Thus, AIDA*'s batch size can be considered $b_{\text{PDB}}=n$. Suppose each $h_{M}$ lookup requires $t_{\text{NN}}$ time. Although $t_{\text{NN}}$ slightly varies with batch size, we assume here that it remains constant across batch sizes. Additionally, we define $t_{\text{copy}}$ as the time required to transfer the batch between CPU and GPU. 

For AIDA*, where all threads perform PDB lookups in parallel without race conditions, the lookup overhead for all threads is comparable to that for a single thread: $t_{\text{PDB}}$ for $n$ lookups is roughly equivalent to $t_{\text{PDB}}$ for one lookup. Now, we make the following assumption regarding batch sizes in Batch IDA*:

\begin{assumption} \label{assumption: batch size}
    The batch size for all $h_{M}$ lookups is always $b_{\text{NN}}$. 
\end{assumption}

We leave out of scope certain conditions in which CPU threads cannot fill the batches, activating the timeout mechanism. Given Assumption \ref{assumption: batch size}, for Batch IDA* and AIDA* to achieve comparable timing performance, the following condition must hold:

\begin{equation} \label{eq1}
    \frac{b_{\text{NN}}*t_{\text{PDB}}}{b_{\text{PDB}}}=\frac{b_{\text{NN}}*t_{\text{PDB}}}{n}=t_{\text{NN}}+2t_{\text{copy}}
\end{equation}

The left-hand side of Equation \ref{eq1} represents the time AIDA* needs to process $b_{\text{NN}}$ nodes, while the right-hand side shows the time for Batch IDA* to process the same batch. A straightforward approach to improving Batch IDA* performance is to increase $b_{\text{NN}}$, which initially yields substantial gains. However, beyond a certain point, further increases offer diminishing returns due to the rise in $t_{\text{copy}}$. Therefore, other solutions to enhance Batch IDA* efficiency include minimizing $t_{\text{NN}}$ and $t_{\text{copy}}$. This can be achieved by employing faster GPUs or optimizing $h_{M}$ using the following objective:

\begin{equation}
\begin{aligned}
    \max_{h_M} \quad & \frac{1}{|S|} \sum_{s \in S} h_M(s), \\
    \text{subject to} \quad & h_M(s) \leq h^*(s), \quad \forall s \in S, \\
    & |h_M| \text{ is minimized}
\end{aligned}
\end{equation}

This optimization balances the heuristic model size while maintaining admissibility and maximizing the average heuristic value.

\subsection{Hardware}

In this section, we assess the performance of Batch IDA* across different hardware configurations, categorized as H1, H2, and H3. Table \ref{table: hardware} provides an overview of these configurations, with their comparative analysis as follows:

\begin{itemize} 
    \item CPU performance: H3 $>$ H1 $>$ H2 
    \item GPU performance: H3 $>$ H1 $>$ H2 
\end{itemize}

\begin{table}[htbp]
\centering
\resizebox{\columnwidth}{!}{
\begin{tabular}{
    c|
    c|
    c
}
\toprule
\headercell{Server} & \multicolumn{1}{c|}{GPU} & \multicolumn{1}{c}{CPU} \\
\cmidrule(l){1-3}
H1 & GeForce RTX 2080 Ti & AMD Ryzen Threadripper 2950X \\
H2 & GeForce GTX TITAN X & Intel Xeon E5-2620 v3 \\
H3 & GeForce RTX A5000 & AMD EPYC 7313 \\
\bottomrule
\end{tabular}
}
\caption{Details of hardware resources.}
\label{table: hardware}
\end{table}

\begin{table}[htbp]
\centering
\begin{tabular}{
    c
    S[table-format=2.2]
    S[table-format=3.2]
    S[table-format=2.2]
    S[table-format=2.2]
}
\toprule
\multicolumn{5}{c}{Average Time(s) Over 50 Instances} \\
\cmidrule(l){1-5}
\headercell{Hardware}  & \multicolumn{2}{c}{SingleGPU Batch IDA*} & \multicolumn{2}{c}{AIDA*}\\
\cmidrule(l){2-3}
\cmidrule(l){4-5}
& RC & STP & RC & STP \\ 
\midrule
  H1      & 5.87           & 0.78            & 2.32         & 0.026       \\
  H2      & 10.78          & 1.01            & 6.81         & 0.045        \\
  H3      & 4.32           & 0.60            & 2.01         & 0.015         \\  
\bottomrule
\end{tabular}
\caption{Performance across different hardware resources.}
\label{table: hardware analysis}
\end{table}

Table \ref{table: hardware analysis} presents the performance of SingleGPU Batch IDA* and AIDA* across these systems. All systems are configured with 8 threads, due to one server having a maximum of 8 threads. Both methods achieve their best performance on H3 as it has the best CPU and GPU. Although further analysis is required to understand the individual impact of CPU and GPU, these results indicate that AIDA* is more sensitive to hardware variations, as demonstrated by the more performance gap between H3 and H2 compared to that of Batch IDA*.

\subsection{Hyperparameter Analysis}

In this section, we present the impact of hyperparameters on the performance of Batch IDA*. For this experiment, we consider the batch size and the number of subtrees per thread, specified by \texttt{workNum} in Algorithm \ref{alg: parallel search}. The timing results for different combinations of hyperparameter values are shown in Figure \ref{fig: hyper-analysis}. In Rubik's Cube there is only about 10\% variation between the best and worst parameters, whereas in STP there is a 1.84$\times$ difference. The STP domain has less balanced subtrees than Rubik's Cube. With too many subtrees, the work is divided out to all threads too quickly, reducing the ability to load-balancing between smaller and larger subtrees. The range of batch sizes did not have a large impact on performance. This is due to the timeout mechanism in place; if batch filling exceeds the timeout, the process is stopped, and the batch is processed with its current size.

\begin{figure}[t]\small
    \centering
        \begin{subfigure}[b]{0.234\textwidth}
        \centering
        \includegraphics[height=3.4cm]{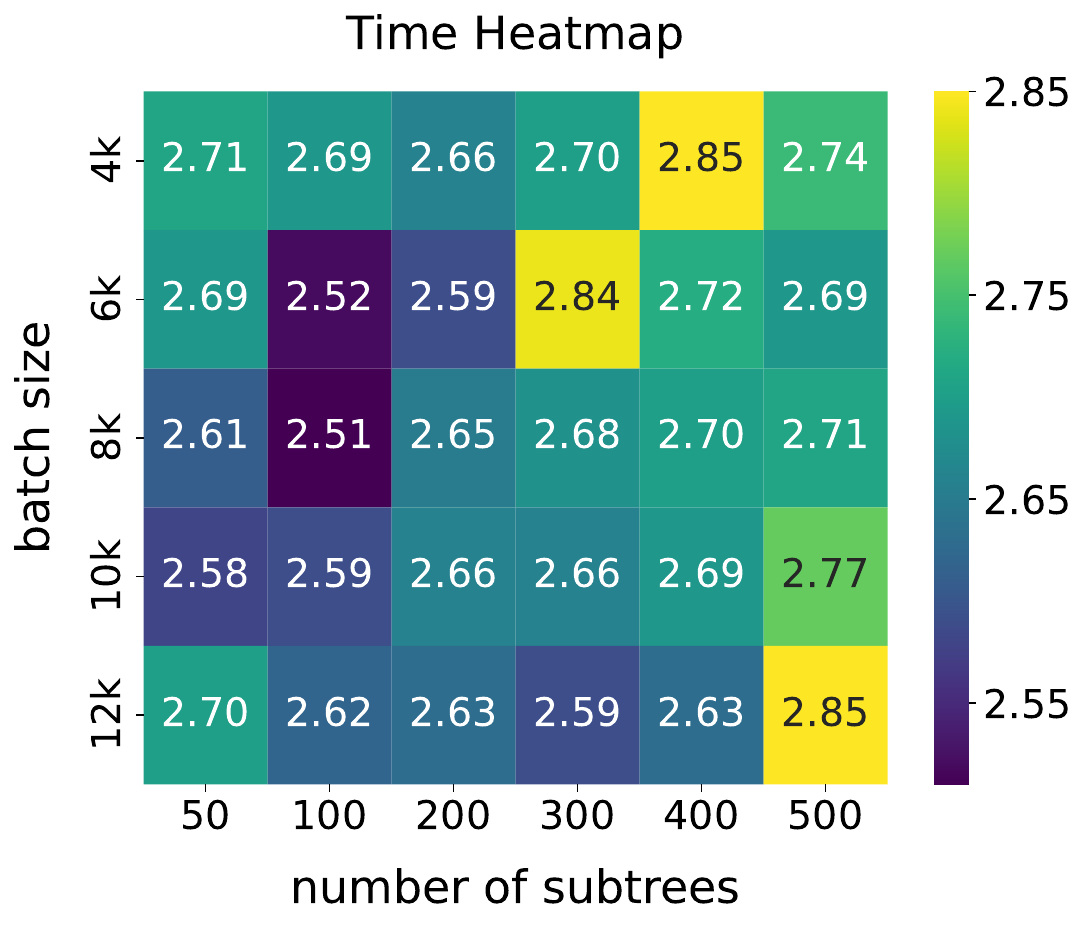}
        \caption{2GPU in RC}
        \end{subfigure}
        \begin{subfigure}[b]{0.234\textwidth}
        \centering
        \includegraphics[height=3.4cm]{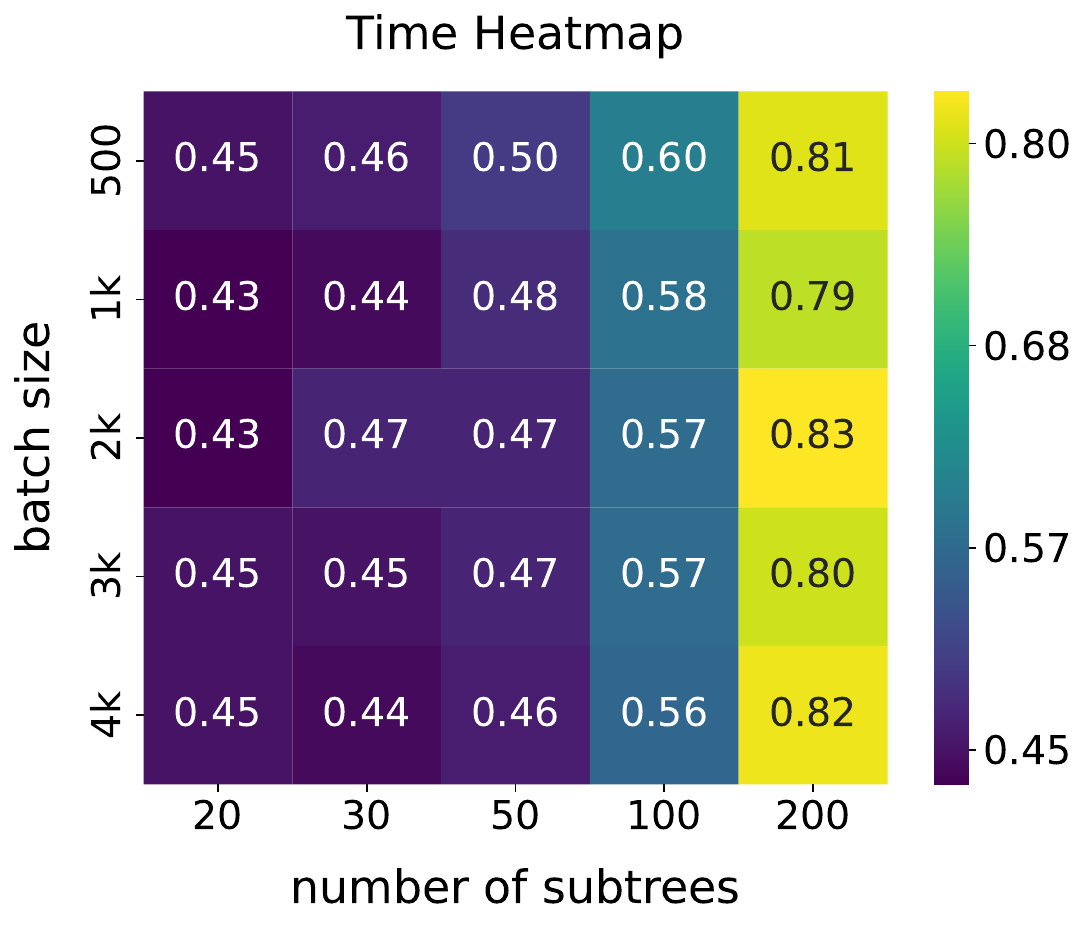}
        \caption{2GPU in STP}
        \end{subfigure}
        \begin{subfigure}[b]{0.234\textwidth}
        \centering
        \includegraphics[height=3.4cm]{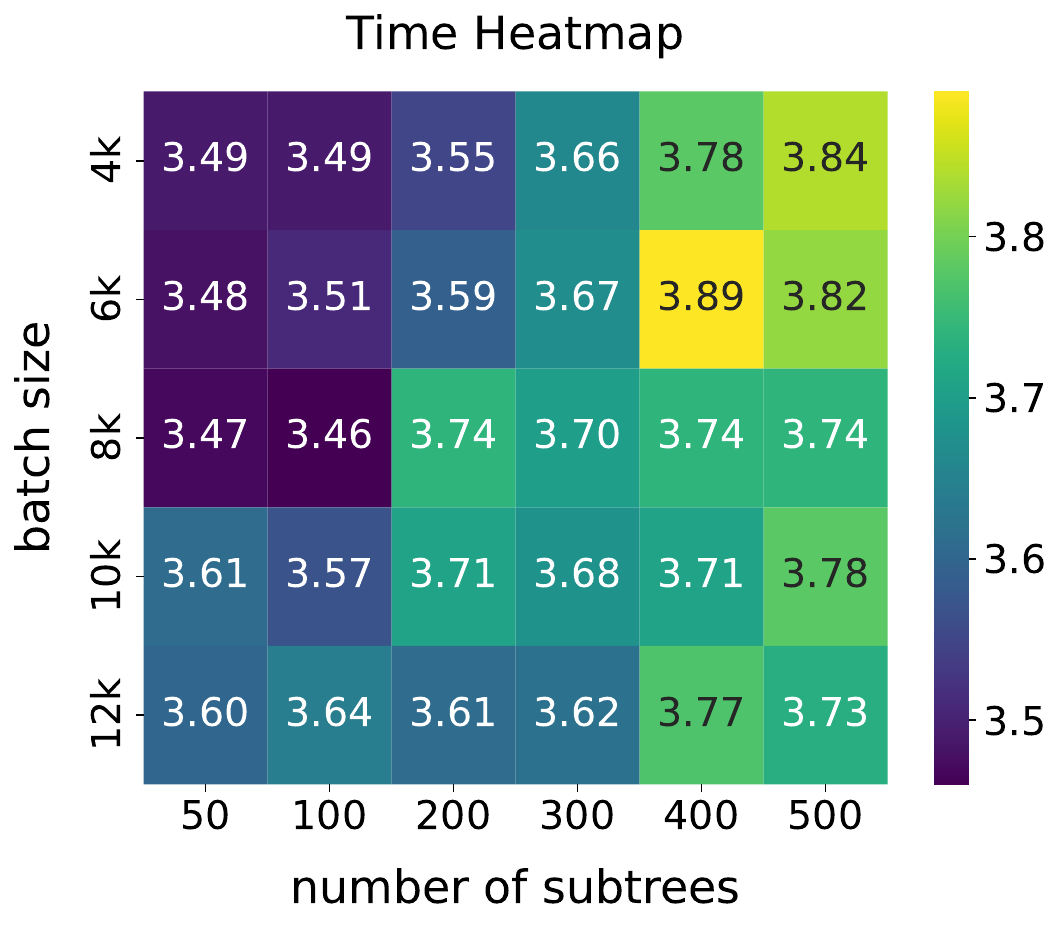}
        \caption{SingleGPU in RC}
        \end{subfigure}
        \begin{subfigure}[b]{0.234\textwidth}
        \centering
        \includegraphics[height=3.4cm]{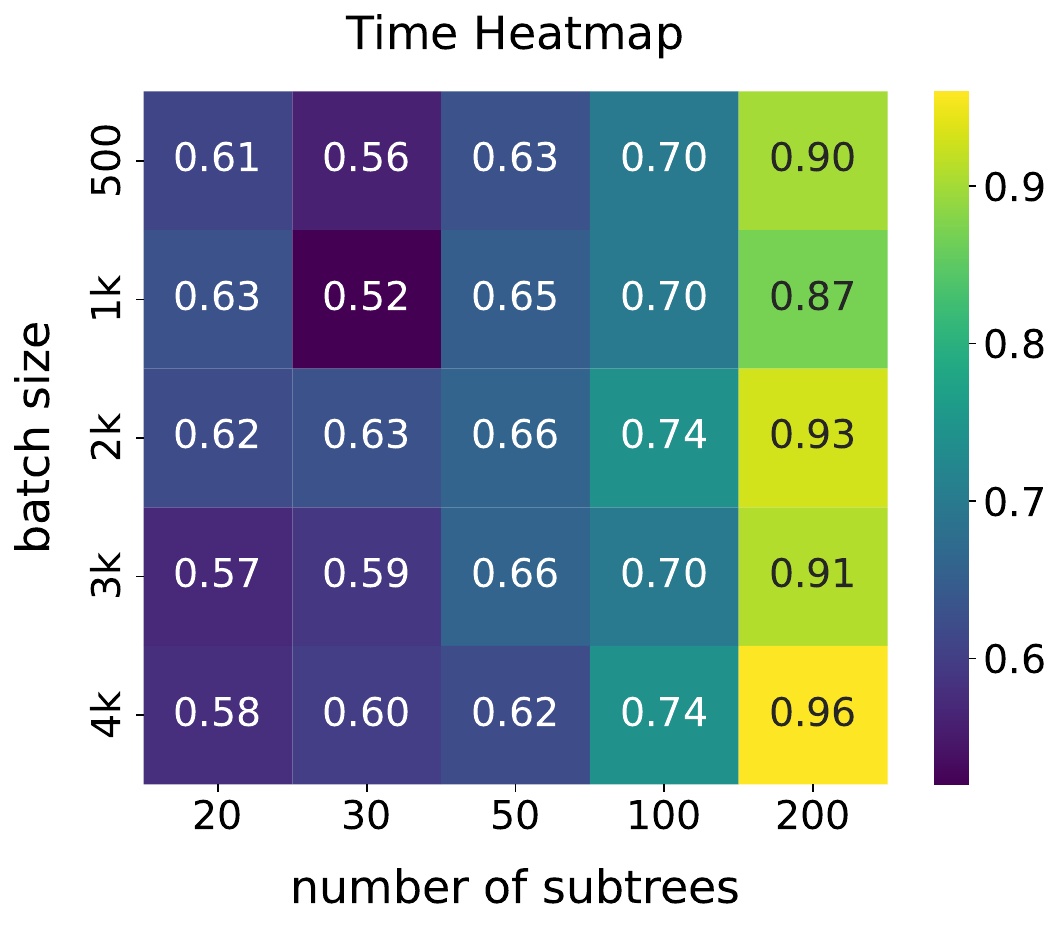}
        \caption{SingleGPU in STP}
        \end{subfigure}
    \caption{Hyperparameter analysis for Batch IDA*.} 
    \label{fig: hyper-analysis}
\end{figure}

\subsection{Proofs of Theoretical Results}

\begin{lemma}\label{lemma1}
    Given an admissible monotone cost function, iterative-deepening- A* will find a solution of least cost if one exists~\cite{korf1985depth}.
\end{lemma}

\begin{lemma}\label{lemma2}
    Let $f$ be an admissible and monotone cost function. Then, the AIDA* algorithm is guaranteed to find an optimal solution, provided that a solution exists.  
\end{lemma}

\begin{proof}

From Lemma \ref{lemma1}, we know that IDA* using the cost function $f$ finds an optimal solution if a solution exists. Therefore, our task is to demonstrate that AIDA*, using the same cost function $f$ as IDA*, expands all the nodes that IDA* expands at each cost threshold. We know that, like IDA*, AIDA* does not increase the cost threshold until all nodes with costs below the current threshold have been expanded. Assume, for the sake of contradiction, that there exists a node $s_{missed}$ that IDA* expands but AIDA* does not. Let us also assume that the initial phase in AIDA* proceeds up to a depth of $d_{init}$. We select a path from the root to $s_{missed}$, denoted as $path$. The node $s_{missed}$ falls into one of two categories:

\begin{enumerate}[label=\arabic*.]
    \item The length of $path$ is less than or equal to $d_{init}$.
    \item The length of $path$ exceeds $d_{init}$.
\end{enumerate}

If $s_{missed}$ falls within the first category, it must be expanded during initial phase, as this phase generates the entire tree up to $d_{init}$ levels. Therefore, in this case, $s_{missed}$ cannot be missed by AIDA*. If $s_{missed}$ falls into the second category, we can express the $path$ as $\{root,s_{1},...s_{d_{init}},...,s_{missed}\}$. The $path$ intersects the tree generated in the initial phase at node $s_{d_{init}}$. We know that, during second and third phases, AIDA* grows the subtree for all leaf nodes of the initial tree, including $s_{d_{init}}$. Since both AIDA* and IDA* use the same cost function $f$, the same cost threshold, and both employ a depth-first search strategy, AIDA* is guaranteed to expand the node $s_{missed}$ during subtree expansion of $s_{d_{init}}$. Thus, AIDA* does not miss any node that IDA* expands. This concludes the
proof that AIDA* will find an optimal solution if one exists.
\end{proof}

\begin{theorem}
    Let $h_{M}$ be an admissible heuristic learned from a PDB heuristic $h$. Then, Batch IDA* using $h_{M}$ is guaranteed to find an optimal solution, assuming that a solution exists.
\end{theorem}

\begin{proof}
    From Lemma \ref{lemma2}, we know that AIDA*, using a PDB heuristic $h$, is guaranteed to find the optimal solution if a solution exists. To extend this guarantee to Batch IDA*, we must demonstrate that CB-DFS, when using the learned heuristic $h_{M}$ derived from $h$, does not omit any node expansions at any cost threshold. The primary difference between Batch IDA* and AIDA* lies in the heuristic computation. Therefore, we need to verify two properties:

    \begin{enumerate}[label=\arabic*.]
        \item \textit{Heuristic Evaluation Completeness}: No node is expanded or discarded without having its heuristic properly evaluated. In CB-DFS, each thread switches to another subtree if the heuristic evaluation of the top node in its work is not yet available. The thread does not advance the subtree until the batch evaluation is complete and the heuristics are ready. This mechanism guarantees that every node is assessed using a valid heuristic $h_{M}$.
        \item \textit{Expansion Condition Alignment}: Given that $h_{M}$ is an admissible heuristic, it inherently serves as a lower bound for the original PDB heuristic $h$:
            \begin{equation*}
                h_{M}(s) \leq h(s) \ \ \ \ \forall s \in S.
            \end{equation*}
            This inequality implies that if a node satisfies the expansion condition under AIDA*, meaning $h(s)+g(s) \leq \textit{threshold}$,  it will also satisfy the expansion condition under the CB-DFS: $h_{M}(s)+g(s) \leq \textit{threshold}$. Therefore, any node that qualifies for expansion in AIDA* will also qualify in Batch IDA*'s CB-DFS.
    \end{enumerate}   
    
    Since both properties hold, Batch IDA* is guaranteed to find an optimal solution, assuming one exists.
\end{proof}

A similar argument extends to BTS, and any algorithm that uses our CPU/GPU-based CB-DFS.

\end{document}